\documentclass[pdflatex,sn-mathphys-num,Numbered]{sn-jnl}

\usepackage{graphicx}
\usepackage{multirow}
\usepackage{amsmath,amssymb,amsfonts,amsthm}
\usepackage{booktabs}
\usepackage{tabularx}
\usepackage{array}
\usepackage{url}
\usepackage{xcolor}
\usepackage{tikz}
\usetikzlibrary{arrows.meta, positioning, calc}
\usepackage{listings}
\usepackage{enumitem}
\hypersetup{hidelinks}

\newcolumntype{Y}{>{\raggedright\arraybackslash}X}

\newtheoremstyle{p17plain}%
  {6pt}{6pt}{\itshape}{0pt}{\bfseries}{.}{ }{}%
\newtheoremstyle{p17definition}%
  {6pt}{6pt}{\normalfont}{0pt}{\bfseries}{.}{ }{}%
\theoremstyle{p17definition}
\newtheorem{definition}{Definition}
\theoremstyle{p17plain}
\newtheorem{proposition}{Proposition}

\newtheorem{corollary}{Corollary}

\renewcommand{\proofname}{Proof}
\makeatletter
\renewenvironment{proof}[1][\proofname]{\par
  \pushQED{\qed}\normalfont \topsep6\p@\@plus6\p@\relax
  \trivlist
  \item[\hskip\labelsep\bfseries
    #1\@addpunct{.}]\ignorespaces
}{\popQED\endtrivlist\@endpefalse}
\makeatother

\begin{document}

\title[Governed Caste Reassignment]{Governed Caste Reassignment in Heterogeneous Swarms:
An Asymmetric-Trust Protocol with Audited Operator Countersignature}

\author[1]{\fnm{Xue} \sur{Qin}}\email{qinxue@me.com}

\author[2]{\fnm{Simin} \sur{Luan}}\email{luansiminiot@gmail.com}

\author*[3]{\fnm{Cong} \sur{Yang}}\email{cong.yang@suda.edu.cn}

\author*[2]{\fnm{Zhijun} \sur{Li}}\email{lizhijunos@hit.edu.cn}

\affil*[1]{\orgdiv{School of Software}, \orgname{Harbin Institute of Technology}, \orgaddress{\city{Harbin}, \country{China}}}

\affil[2]{\orgdiv{School of Computer Science and Technology}, \orgname{Harbin Institute of Technology}, \orgaddress{\city{Harbin}, \country{China}}}

\affil[3]{\orgdiv{School of Future Science and Engineering}, \orgname{Soochow University}, \orgaddress{\city{Suzhou}, \country{China}}}

\abstract{In heterogeneous robot swarms, caste reassignment is a
high-frequency runtime event: battery depletion forces a
logistics-caste robot into a low-power patrol caste; a payload
limit forces a heavy-lift caste robot into an observation caste.
Existing approaches treat reassignment as an internal allocation
algorithm (auction, consensus bundle, behaviour-tree) and do
not expose the reassignment event to external authority. We argue
that for regulated embodied deployments a caste change that
elevates a robot's privilege envelope is a governance event that
must be auditable and externally authorised. We propose an asymmetric-trust
protocol in which auto-tightening reassignments (going to safer
or lower-privilege castes) are admitted automatically while
bounded relaxation (going to higher-privilege castes) requires
operator countersignature. The protocol carries a signed
cause-chain documenting why each reassignment was needed; the
chain is a hash-chained, Merkle-committed audit log verifiable
offline against the operator's key. We evaluate a reference
implementation with real Ed25519 signatures and a hash-chained
Merkle audit log, in simulation over fleets up to $100$ robots
with a parameterized robot-class timing point at $10$ robots (slower
signatures, heavier wireless variance; not a hardware measurement):
auto-tightening completes in single-digit to low-double-digit
milliseconds and the audit gate adds single-digit to low-double-digit
milliseconds over an ungoverned reassignment (growing with fleet size). The protocol refuses all four explicit attack
patterns (caste laundering, repeated-relaxation privilege escalation,
operator impersonation, and cause-chain forgery) by construction; a
partially-governed baseline isolates which gate stops which attack, a
real offline auditor reconstructs and verifies every record and rejects
tampering, and forged tighten records never reach the Byzantine quorum.
The construction generalises a single-agent persona-mutation governance
gate to swarm-level caste governance.}


\maketitle

\begin{figure}[t]
\centering
\includegraphics[width=0.9\linewidth]{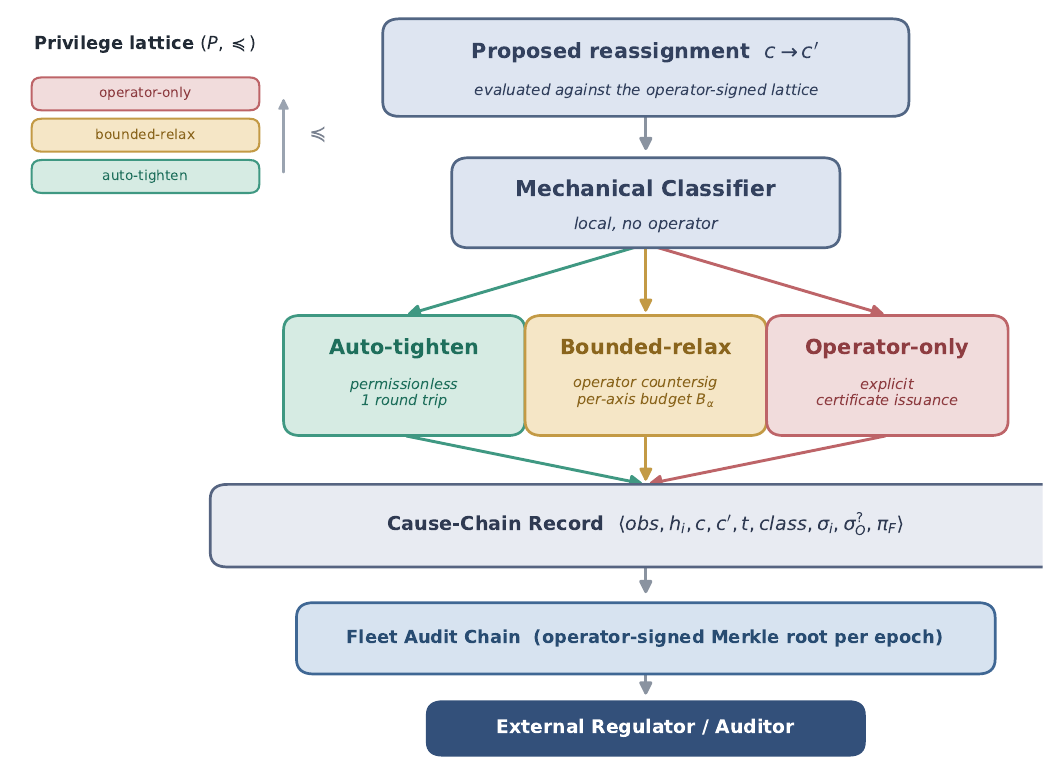}
\caption{Governed caste reassignment, overview. A proposed
caste transition $c \to c'$ is classified locally and
mechanically against the operator-declared privilege lattice
$(P, \preceq)$ into one of three classes (auto-tighten,
bounded-relax, operator-only). Each class is admitted through a
different authorisation path: tightening is permissionless,
bounded relaxation requires an operator countersignature
against a per-axis bound $B_\alpha$, and operator-only
admission requires explicit certificate issuance. Every
transition produces a cause-chain record carrying the
triggering observation, the proposing robot's identity hash,
the source and target castes, the class, the robot signature,
and the operator countersignature (if any), appended to the
fleet-level audit chain whose per-epoch Merkle root the operator
signs and which is presentable to an external regulator.}
\label{fig:hero}
\end{figure}

\section{Introduction}\label{sec:intro}

A heterogeneous robot swarm operating in a regulated environment
(a hospital ward fleet, a factory cobot pool, a defence
installation) must continuously reassign \emph{caste}, our term
for a role bound to a capability profile, as battery levels
fall, payloads change, and operational priorities shift. We use
``caste'' and ``role'' interchangeably, preferring ``caste'' for
its biological connotation of a capability-bound role.
Multi-agent systems have a long tradition of treating
reassignment as a first-class problem. Role allocation and
reallocation were formalised by Nair, Tambe and
Marsella~\citep{nair2003roleallocation} at AAMAS 2003 and
surveyed comprehensively by Campbell and
Wu~\citep{campbell2011roleallocation}. Organisational MAS adds a
deontic dimension binding permissions and obligations to
roles, as in the MOISE+ family~\citep{hubner2002moiseplus}, and
the broader organisational-paradigm
landscape~\citep{horling2004orgparadigms}. Coalition
formation~\citep{shehory1998coalition} and market-based
multi-robot coordination~\citep{dias2006marketmrta} offer
algorithmic substrates, with the MRTA taxonomy~\citep{gerkey2004taxonomy}
and its extension iTax~\citep{korsah2013itax} as canonical
classifications. The closest analogue to our human-anchored
authorisation is adjustable
autonomy~\citep{scerri2002adjustableautonomy}, which models
when an agent transfers a decision to an external human
authority via a Markov-decision-process-style transfer-of-control
policy. Normative MAS~\citep{boella2006normativemas} models
roles and obligations as norms, with explicit monitoring and
sanctioning.

This paper argues that for regulated deployments a caste change
that elevates a robot's privilege envelope is a governance
event that must be both auditable and externally authorised, and
that the existing constructions above do not jointly supply
three properties our setting requires: (i) \emph{privilege
monotonicity} as a structural invariant (tightening is
permissionless, relaxation is gated); (ii) \emph{cryptographic,
offline-verifiable} auditability of every transition; and
(iii) a \emph{per-axis bounded budget} on relaxation that an
external regulator can inspect without trusting any internal
allocation algorithm. Adjustable autonomy supplies the
transfer-of-control mechanism but not the cryptographic
audit-chain anchoring, role reallocation supplies the algorithmic
choice but not the asymmetric gating, and normative MAS supplies
the norm-and-sanction vocabulary but not the privilege-lattice
discipline.

We propose an asymmetric-trust protocol pattern that ports a
single-agent persona-mutation governance gate~\citep{aeros-p15}
into the swarm setting and anchors it in the fleet-level
federation primitive of~\citep{aeros-p16}. The construction sits
within a programme of runtime governance for embodied agents
that previously factored admission, policy enforcement,
recovery, and human takeover out of the agent into a separable
layer~\citep{aeros-p3}; the present paper extends the same
posture to caste change in heterogeneous fleets. The protocol
classifies each proposed reassignment as auto-tighten,
bounded-relax, or operator-only, and admits each class through a
different authorisation path. The delta over prior MAS work is
the combination of (i)--(iii) above, not the idea that role
change should be externally governed.

\paragraph{Contributions.} This paper makes five contributions.

\begin{enumerate}[leftmargin=*, itemsep=2pt]
\item A \emph{caste reassignment taxonomy} (\S\ref{sec:taxonomy})
  classifying every reassignment as one of three governance
  classes (auto-tighten, bounded-relax, operator-only), motivated
  by eusocial-insect caste-plasticity
  analogues~\citep{wilson1971insectsocieties,robinson1992agepolyethism}
  but reframed as a software-pattern problem.

\item An \emph{asymmetric-trust protocol} (\S\ref{sec:protocol})
  lifting a single-agent asymmetric-trust gate~\citep{aeros-p15}
  to swarm-level caste reassignment. Tightening is
  permissionless; bounded relaxation requires operator
  countersignature; operator-only castes require explicit
  certificate issuance. Every transition is auditable.

\item A \emph{cause-chain construction} (\S\ref{sec:causechain})
  documenting reassignment provenance, integrated with the
  fleet-level audit chain of the federation primitive
  of~\citep{aeros-p16}. The chain supports forensic
  reconstruction of any reassignment from the audit chain alone.

\item An \emph{evaluation} (\S\ref{sec:eval}) of a reference
  implementation with real Ed25519 signatures and a hash-chained
  Merkle audit log, in simulation over fleets up to $100$ robots (with
  a parameterized robot-class timing point at $10$ robots). It reports
  single-digit to low-double-digit millisecond tightening latency and a
  single-digit to low-double-digit millisecond governance overhead
  (growing with fleet size), and it isolates the gate each
  attack requires via a partially-governed baseline, exercises the
  offline auditor and the Byzantine quorum directly, and reports the
  measured cost of the governed decision.

\item A \emph{distributed audit layer} (\S\ref{sec:replication}) that
  replicates the admitted audit records across $N$ per-member logs with a
  quorum-committed total order, proves an agreement and fork-exclusion
  guarantee (Proposition~\ref{prop:fork-exclusion}), and is evaluated both in
  a property simulation and as a real deployment of $N$ separate processes
  over TCP sockets with a real Byzantine equivocator (\S\ref{sec:rq5}), on
  which every honest replica agrees and detects the equivocation and no fork
  commits. This is a safety layer over already-admitted records, not a new
  consensus protocol; it composes the standard quorum-certificate and audit-log
  primitives for the governed-caste setting. The agreement guarantee is a safety
  property for running nodes (the real run is single-host loopback, and
  crash-recovery durability and multi-host networking are future work).
\end{enumerate}

\paragraph{Scope of the contribution.} We are precise about what this
paper offers. The contribution is a governance \emph{pattern} for caste
reassignment, together with a reference implementation that realises the
full cryptographic substrate (per-principal Ed25519 keys, an
operator-signed identity manifest, a hash-chained Merkle audit log with a
persisted Byzantine-quorum certificate, an offline auditor, and a
replicated $N$-member audit layer with quorum-committed total order and
fork exclusion) and evaluates it in simulation. The novelty is the specific
combination (i)--(iii) above, not any single primitive: role governance,
access-control lattices, threshold signatures, and audit logs are all
classical. The replicated-log layer with fork exclusion is implemented and
evaluated directly, both in a property simulation at scale and as a real
multi-process deployment over TCP sockets with a real Byzantine equivocator
process (\S\ref{sec:replication}, \S\ref{sec:rq5}); what remains out of scope
is a physical-robot deployment, a multi-host wide-area network, and the
asynchronous-consensus concerns (network asynchrony, view changes, and
liveness under asynchrony) that any production replication would layer beneath
the safety property we establish (\S\ref{sec:discussion}).

\section{Related Work}\label{sec:related}

\paragraph{Multi-Robot Task Allocation.}
The MRTA taxonomy~\citep{gerkey2004taxonomy} and its extension iTax
of Korsah et~al.~\citep{korsah2013itax} are the canonical references
for the internal-allocation framing, with coalition
formation~\citep{shehory1998coalition} and market-based multi-robot
coordination~\citep{dias2006marketmrta} as algorithmic substrates.
Recent dynamic allocation extends the line with behaviour-tree
capabilities~\citep{behaviortree2024mrta} and global-games
formulations~\citep{globalgames2025mrta}, and the
Swarmanoid~\citep{dorigo2013swarmanoid} lineage demonstrates
heterogeneous-swarm hardware. These constructions optimise an
internal objective (throughput, makespan, energy budget) and treat
reassignment as an algorithmic decision; none expose the reassignment
event to external authority. Caste reassignment is orthogonal: the
protocol governs role \emph{change} rather than coalition choice, and
composes above any allocator that respects the privilege lattice.

\paragraph{Role Allocation and Reallocation.}
The closest body of MAS work is role allocation and reallocation
in agent teams. Nair, Tambe and
Marsella~\citep{nair2003roleallocation} formalise role
reallocation as conditions change, and provide the practical
analysis that subsequent work builds on. Campbell and
Wu~\citep{campbell2011roleallocation} survey the field and
articulate its open issues. More recent work brings a human into
the reallocation loop: Al-Hussaini et
al.~\citep{alhussaini2024reallocation} generate task-reallocation
suggestions for a human supervisor to approve under
contingencies. These constructions choose roles for agents under
an objective (efficiency, makespan, robustness), and even when a
human approves the choice they do not bind the role change to an
offline-verifiable authorisation record: the decision is the
final word and leaves no cryptographic provenance. Our
construction layers a governance gate above the choice, leaving
the algorithm to propose and adding a cryptographic,
privilege-monotone authorisation surface.

\paragraph{Organisational MAS and Reorganisation.}
The deontic dimension of organisational MAS attaches
permissions and obligations to roles. The MOISE+
model~\citep{hubner2002moiseplus} specifies structural,
functional, and deontic facets of agent organisations;
the broader paradigm
survey of Horling and Lesser~\citep{horling2004orgparadigms}
positions teams, coalitions, hierarchies, and federations.
The privilege lattice of \S\ref{sec:taxonomy} can be read as a
runtime, cryptographically enforceable instance of the
\emph{permission} facet of MOISE+'s deontic specification: it
governs which actions a caste may perform, while the richer
obligation relations of MOISE+ are out of scope. Closest to our
framing of role change as a managed, authorised event is the
MOISE+ \emph{reorganisation} framework of H\"ubner et
al.~\citep{hubner2004reorg}, which treats organisational change
as a first-class cooperative process; we add the asymmetric
authorisation gate and the cryptographic audit record that
reorganisation leaves unspecified.

\paragraph{Adjustable Autonomy and Human Supervisory Control.}
The closest conceptual analogue to operator countersignature
is adjustable autonomy. Scerri, Pynadath and
Tambe~\citep{scerri2002adjustableautonomy} model the decision
of when an agent transfers control to an external human
authority as an MDP-style policy choice, and the human
supervisory control of multi-robot teams surveyed by Chen et
al.~\citep{chen2011supervisory} establishes the
operator-in-the-loop setting our countersignature lives in. The
bounded-relax path of \S\ref{sec:protocol} can be viewed as an
asymmetric, privilege-monotone instance of transfer-of-control
where the trigger is fixed (any relaxation along an
operator-bounded axis), the authority is unambiguous (the
operator's signing key), and the act produces a cryptographically
verifiable audit record. Our delta over MDP-style adjustable
autonomy is the offline-verifiable cause-chain and the
per-axis budget construction.

\paragraph{Normative MAS and Electronic Institutions.}
Boella, van der Torre and Verhagen~\citep{boella2006normativemas}
introduce normative multi-agent systems and the
norm-monitoring-and-sanctioning machinery that follows, and the
electronic-institutions line~\citep{esteva2001einst}, with its
runtime norm enforcement~\citep{garciacamino2005norms}, supplies
the canonical framework for imposing external authority on agent
interactions. Caste reassignment under operator countersignature
is a regulative-norm instance: the norm constrains which
transitions are admissible, and the operator signature is the
monitoring authority. The cause-chain is the offline-verifiable
\emph{monitoring and evidence} record, not a sanction; what
follows from a detected violation (the sanctioning step) is a
deployment-policy decision we leave out of scope. Our
contribution is the cryptographic realisation of this monitoring
layer in a multi-robot setting.

\paragraph{Multi-Robot Authorisation Infrastructure and Blockchain Swarms.}
SROS2~\citep{sros2_2022} provides per-node certificate
issuance and signed transport, with attribute-based access
control extensions~\citep{abac2023multirobot} for per-robot
authorisation, and the privilege lattice itself instantiates the
lattice-based access-control tradition~\citep{sandhu1993lattice}.
The blockchain-swarm
literature~\citep{strobel2020blockchain,castelloferrer2024review}
addresses inter-robot trust under a flat-Byzantine assumption
without operator anchor. Caste handling, where present in those
constructions, sits inside internal allocation and is not
exposed to external authority. Our construction operates above
the SROS2 substrate and outside the blockchain-internal-trust
posture, with operator-authority caste authorisation as its
distinguishing axis.

\paragraph{Biological Caste Plasticity.}
Eusocial insect colonies exhibit caste plasticity at the level of
individual workers~\citep{wilson1971insectsocieties,
robinson1992agepolyethism,bonabeau1999swarmintelligence}: a worker
bee progresses from nurse to forager as it ages (age polyethism)
and may shift caste under colony-level demands. The biological
process is homeostatic and developmental; we adopt the caste
vocabulary and the empirical observation of frequent reassignment,
but the authorisation construct is engineered rather than
inherited.

\paragraph{Programme Context.}
The protocol generalises a single-agent persona-mutation
governance gate~\citep{aeros-p15} from one robot to a fleet,
and is anchored in the federation primitive
of~\citep{aeros-p16}. The broader runtime-governance posture
the construction inherits (admission, policy, exception
monitoring, recovery, and human takeover as a separable layer
above the agent) is established in~\citep{aeros-p3} at the
single-agent level; the present paper extends that posture to
the multi-robot setting. To delineate what is reused from what is
new: the single-agent asymmetric gate, the operator-signed
identity manifest, and the BLS-aggregated Merkle audit chain are
taken from~\citep{aeros-p15,aeros-p16}; this paper contributes
the lattice-structured \emph{per-axis} relaxation budget anchored
at the operator-reset baseline (Definition~\ref{def:lattice},
Proposition~\ref{prop:laundering}), the swarm-quorum lift of the
gate (Proposition~\ref{prop:tighten-quorum}), the caste-laundering
containment argument, and the upstream-signature separation in the
cause-chain record. The properties this paper proves are proved
here against the model of \S\ref{sec:fleetmodel}, not inherited.

\section{Method}\label{sec:method}

We present the construction in four parts: a fleet and threat model,
a privilege-lattice taxonomy of caste transitions, the three-path
asymmetric-trust protocol, and the cause-chain record that makes
every transition auditable.

\subsection{Fleet and Threat Model}\label{sec:fleetmodel}

We consider a fleet of $N$ robots communicating over authenticated,
integrity-protected channels. The per-node certificates and signed
transport of SROS2~\citep{sros2_2022} supply message authentication but
not agreement, so the agreement layer is stated here rather than
inherited. Among the $N$ members at most $f$ may be Byzantine, with the
standard bound $f = \lfloor (N-1)/3 \rfloor$; honest members follow the
protocol and Byzantine members deviate arbitrarily. A \emph{tighten}
record is admitted only when a quorum of $\lceil (N+f+1)/2 \rceil$
members acknowledge it, the classical authenticated-Byzantine quorum
under which any two quorums intersect in at least one honest
member~\citep{lamport1982byzantine}; bounded-relax and operator-only
records are instead gated by the operator countersignature rather than a
fleet quorum (\S\ref{sec:protocol}), and the reference implementation
enforces the quorum on exactly the tighten path
(Proposition~\ref{prop:tighten-quorum}). This quorum is a local admission
rule that supplies honest endorsement of a tighten record; the replicated
layer that extends it to a fork-free total order across the fleet, and its
agreement and fork-exclusion guarantee, are given in \S\ref{sec:replication}
(Proposition~\ref{prop:fork-exclusion}) and evaluated in \S\ref{sec:rq5}. An
external \emph{operator} holds a
signing key whose public half every member knows; the operator is the
sole authority for relaxation countersignatures and operator-only
certificates. We use ``asymmetric trust'' in a structural, not
reputational, sense: tightening is trusted to the fleet quorum, and
relaxation is trusted only to the operator. This differs from the
computational trust-and-reputation models surveyed
by~\citet{pinyol2013trust}, which estimate peer trustworthiness from
observed behaviour. Here trust placement is fixed by the privilege
direction of the transition and does not depend on any peer's history.

\paragraph{Threat model.}\label{sec:threats}
We assume an adversary attempting to misuse the caste
construction in four explicit patterns:

\begin{enumerate}[leftmargin=*, itemsep=1pt]
\item \emph{Caste laundering.} An adversary-controlled robot
  proposes a tighten (to a lower-privilege caste) and then,
  before the audit chain commits, proposes a relax of the same
  magnitude back to (or beyond) the original caste, hoping the
  net change appears null while the relaxation is unaudited.
\item \emph{Repeated-relaxation privilege escalation.} An
  adversary-controlled robot repeatedly proposes small bounded
  relaxations within an operator-pre-committed envelope, each
  of which the operator countersigns under a fast path, but
  whose cumulative effect exceeds the operator's intended
  envelope.
\item \emph{Operator impersonation.} An adversary forges a
  countersignature on a relaxation request, hoping fleet
  members accept the forged signature as authoritative.
\item \emph{Cause-chain forgery.} An unauthorized party (here a
  registered robot or the adversary, not the designated sensor) signs a
  cause record (e.g., ``battery low at $t$'') and submits it as the
  triggering observation, hoping to justify an unwarranted tighten or
  relax. The protocol's defence is provenance, not physical truth: it
  establishes that the triggering observation carries a valid signature
  from an \emph{authorized sensor}, so a self-signed or adversary-signed
  cause is rejected. It does not certify that a correctly-signed
  observation is physically true, fresh, or unique; a compromised
  authorized sensor and stale-observation replay are separate threats
  outside the scope stated below.
\end{enumerate}

The four patterns above scope the \emph{admission-surface} claims; audit-log
equivocation across replicas is handled separately by the distributed layer
(\S\ref{sec:replication}), which detects a Byzantine equivocator and excludes
forks (evaluated in \S\ref{sec:rq5}). Threats outside this combined set
(signature replay, stale observations, operator-reset abuse, and a
compromised but authorized sensor) are material for a full deployment but are
not modelled or refuted here; we return to them in \S\ref{sec:discussion}.
The evaluation combines conformance and gate-removal ablations (RQ3) with a
randomized fuzz adversary that mutates otherwise-valid records
(\S\ref{sec:rq3b}) and a distributed Byzantine-equivocation experiment
(\S\ref{sec:rq5}), so it is not purely a fixed-case conformance study.

\paragraph{Cryptographic substrate.}
The construction is self-contained on standard primitives. Each
principal (every robot, the operator, and each sensor) holds an
Ed25519 signing key~\citep{bernstein2012ed25519}; an operator-signed
\emph{identity manifest} binds each robot to its public key and records
the privilege lattice (\S\ref{sec:taxonomy}). The audit chain is an
append-only log, hash-chained for tamper evidence and committed per
epoch to a Merkle root~\citep{merkle1988signature} that the operator
signs; an inclusion proof binds any record to that root, and an external
auditor verifies the chain, the proof, and every embedded signature
using only public keys. The reference implementation
(\S\ref{sec:eval}) realises exactly this substrate. Where a fleet wishes
to compress many acknowledgements into one, the federation primitive
of~\citep{aeros-p16} can aggregate signatures (for example via
BLS~\citep{boneh2004bls}); we treat aggregation as an optional
optimisation and do not depend on it for any property proved here.

\subsection{Caste Reassignment Taxonomy}\label{sec:taxonomy}

We first define the privilege lattice and its relaxation
magnitude metric, then partition the space of caste transitions
into three governance classes by their privilege effect.

\begin{definition}[Privilege lattice]\label{def:lattice}
Let $\Sigma$ be a universe of actions a robot can request and
$\mathcal{A}$ a finite set of named \emph{relaxation axes}
(e.g., battery budget, payload weight, mission risk). A
\emph{privilege lattice} is a triple
$\mathcal{L} = (P, \preceq, \mathit{actions})$ where $P$ is the
set of castes, $\preceq$ is a partial order on $P$, and
$\mathit{actions}: P \to 2^\Sigma$ is monotone in $\preceq$,
i.e., $c \preceq c'$ implies
$\mathit{actions}(c) \subseteq \mathit{actions}(c')$.
For each axis $\alpha \in \mathcal{A}$ the operator declares at
deployment time a magnitude function
$\Delta_\alpha: P \times P \to \mathbb{R}_{\ge 0}$ such that
$\Delta_\alpha(c, c') > 0$ iff the transition $c \to c'$ relaxes
along axis $\alpha$. The lattice and the magnitude functions
together are part of the operator-signed identity manifest
(\S\ref{sec:fleetmodel}).
\end{definition}

The partial order $\preceq$ with monotone $\mathit{actions}$ is the
multi-robot, runtime instance of a lattice-based access-control
model~\citep{sandhu1993lattice,sandhu1996rbac}: a caste plays the role
of a security label and tightening is a downgrade along the lattice.
The magnitude $\Delta_\alpha$ is the metric the operator bounds. At
deployment the operator pre-commits an \emph{envelope}
$E \subseteq \mathcal{A}$ of relaxable axes and a \emph{per-axis bound}
$B_\alpha \in \mathbb{R}_{\ge 0}$ for each $\alpha \in E$; the triple
$(E, \{B_\alpha\}, \Delta_\alpha)$ is the bounded-relax authorisation
state used in Section~\ref{sec:protocol}. The bound caps the elevation
of the \emph{target} caste above the most recent
\emph{operator-reset baseline} $c_\star$: a target $c'$ is admissible
along axis $\alpha$ iff $\Delta_\alpha(c_\star, c') \le B_\alpha$. This
is a fixed envelope, not a depleting counter; because it is charged
from $c_\star$ rather than the robot's current caste, no transient
tighten can free headroom (Proposition~\ref{prop:laundering}) and no
sequence of relaxations can exceed it
(Proposition~\ref{prop:relax-containment}). A separate, genuinely
depleting \emph{ratification token} governs only the fast path of the
bounded-relax protocol and is defined in Section~\ref{sec:protocol}. For battery the operator might
declare $\Delta_{\text{battery}}(c, c') = \max(0, \text{busy}(c') -
\text{busy}(c))$ in expected watt-hours per shift; for payload it might
be the additional maximum weight $c'$ may carry. The protocol is
agnostic to the units; it depends only on $\Delta_\alpha$ being
operator-signed and bounded.

\begin{definition}[Caste reassignment classes]\label{def:classes}
Let $O \subseteq P$ be the set of castes the operator has explicitly
reserved for case-by-case authorisation. A caste transition $c \to c'$
is classified, in this priority order, as:
\begin{itemize}[leftmargin=*, itemsep=1pt]
\item \emph{Operator-only} if $c' \in O$.
  Examples: medication-delivery caste in a hospital;
  weapons-bearing caste in defence; passenger-transport caste
  in a logistics fleet.
\item \emph{Auto-tighten} if $c' \notin O$ and $\mathit{actions}(c')$
  is a strict subset of $\mathit{actions}(c)$ (every action $c'$ can
  perform, $c$ can perform; some actions $c$ can perform, $c'$ cannot).
  Examples: logistics $\to$ patrol; manipulator $\to$ observer.
\item \emph{Bounded-relax} if $c' \notin O$, the transition is not a
  tighten, and every axis it relaxes along lies in the
  operator-pre-committed envelope $E$.
  Examples: patrol $\to$ logistics within the battery-budget
  threshold; observer $\to$ light manipulator within the
  payload-weight threshold.
\item \emph{Operator-only} (default) in every remaining case: a
  transition between incomparable castes, or one that relaxes along an
  axis outside $E$, falls here. This makes the classification total: the
  safe default for anything the operator did not pre-authorise is to
  route it through explicit operator admission.
\end{itemize}
\end{definition}

The privilege lattice is operator-declared at deployment time
as part of the identity manifest~\citep{aeros-p16}. Each caste
is annotated with the action set it can perform and the axes
along which it is relaxable. A pure relabel that changes neither
the action set nor any relaxation axis is, by
Definition~\ref{def:classes}, neither tighten nor bounded-relax
and routes to operator-only by default; a privilege-neutral caste
change therefore still produces an audited operator decision. The classification of a proposed
transition is mechanical: a fleet member can compute it locally
given the lattice and does not need to consult the operator to
decide whether a transition is tighten, bounded-relax, or
operator-only. What the operator countersignature gates is the
\emph{relaxation} (and not the classification), and what the
operator certificate issues is the \emph{operator-only}
admission (and not the classification).

The taxonomy is asymmetric. Tightening is structurally safer
than relaxing, and the protocol exploits this asymmetry to keep
the operator's authorisation cadence tractable: tightening
events do not consume operator attention, only relaxation does.

\subsection{Asymmetric-Trust Protocol}\label{sec:protocol}

The protocol has three paths corresponding to the three classes
of Definition~\ref{def:classes}. Each path produces an
auditable cause-chain record (\S\ref{sec:causechain}) anchored
in the fleet-level audit chain. The protocol steps below describe
the distributed form in which each member appends to a local
replica. The latency and attack experiments (RQ1--RQ4) run this
chain as a single logical append-only log with the same hash-chain
and Merkle commitment; the fully replicated $N$-member form, with
quorum-committed total order and fork exclusion, is specified in
\S\ref{sec:replication} and evaluated separately in \S\ref{sec:rq5}.

\paragraph{Tighten path.}
When a robot $i$ proposes an auto-tighten transition $c \to c'$,
the protocol completes in one round trip among fleet members:

\begin{enumerate}[leftmargin=*, itemsep=1pt]
\item Robot $i$ signs a tighten-request record:
  $\langle \mathit{cause},\ c,\ c',\ t,\ \sigma_i \rangle$.
\item Each fleet member verifies that $c \to c'$ is auto-tighten
  under the privilege lattice (Definition~\ref{def:classes}).
\item Members append the record to their local audit chain and
  return an acknowledgement.
\item Robot $i$ adopts caste $c'$ on receiving a quorum of
  acknowledgements.
\end{enumerate}

Tightening does not require operator countersignature. The
correctness property is that tighten is monotone in the
privilege lattice: a sequence of tighten transitions never
results in higher privilege than the starting caste.

\paragraph{Bounded-relax path.}
When robot $i$ proposes a bounded-relax transition $c \to c'$
within an operator-pre-committed envelope, the protocol
completes in two round trips:

\begin{enumerate}[leftmargin=*, itemsep=1pt]
\item Robot $i$ signs a relax-request record carrying the
  triggering observation (e.g., battery state, payload
  measurement, time elapsed since last relaxation).
\item Each fleet member verifies that the request is within the
  envelope: every axis it relaxes along is on the operator's
  pre-committed list, and the elevation of the target caste above
  the operator-reset baseline is within the per-axis bound
  $B_\alpha$ (Definition~\ref{def:lattice}).
\item The operator countersigns the record. This takes a
  \emph{short-circuit form} when a pre-issued \emph{ratification
  token}, a finite depleting headroom of relaxation magnitude,
  can cover the request; the token is then decremented. When the
  token cannot cover the request the countersignature takes a
  \emph{live form} (an interactive operator signature) and the
  operator reissues the token. The token gates latency, not
  security: the envelope bound of step 2 holds on every admitted
  relaxation regardless of the token's state.
\item Fleet members append the countersigned record to the
  audit chain; robot $i$ adopts caste $c'$.
\end{enumerate}

The envelope check on step 2 is what prevents
repeated-relaxation escalation (threat 2 of
\S\ref{sec:threats}): because every admitted target must lie
within $B_\alpha$ of the operator-reset baseline, no sequence of
relaxations, however long, reaches a caste beyond the envelope,
and only an operator reset moves the baseline.

\paragraph{Operator-only path.}
When robot $i$ requires admission into an operator-only caste
$c'$, the protocol degenerates to an explicit operator issuance:

\begin{enumerate}[leftmargin=*, itemsep=1pt]
\item Robot $i$ signs a request record naming the operator-only
  caste and the justification.
\item The operator inspects the request out of band and either
  issues a fresh operator-signed caste binding (an operator
  countersignature over the record) or denies the request.
\item If issued, the operator-signed binding supersedes the
  robot's previous caste and is recorded in the audit chain. The
  robot adopts $c'$.
\end{enumerate}

No automatic admission path exists for operator-only castes.
Their distinguishing engineering property is that the operator
is in the synchronous critical path of admission. The reference
implementation models the binding as this operator countersignature;
certificate expiry and revocation follow standard public-key
practice and are out of scope.

\begin{figure}[t]
\centering
\includegraphics[width=\linewidth]{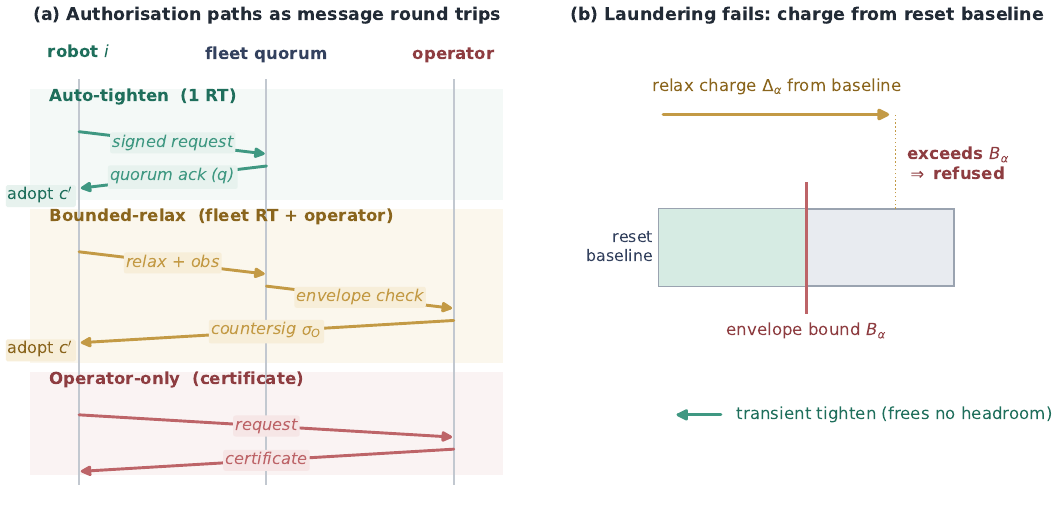}
\caption{Left (a): the three authorisation paths as message round
trips. Auto-tighten completes in one fleet round trip and a quorum of
$\lceil (N{+}f{+}1)/2 \rceil$ acknowledgements; bounded-relax adds an
operator countersignature ($\sigma_O$); operator-only is an explicit
certificate issuance. Right (b): the per-axis budget meter. The relax
charge $\Delta_\alpha$ is measured from the operator-reset baseline, so
a transient tighten frees no headroom and a relax that would exceed
$B_\alpha$ is refused. This is the mechanism behind
Proposition~\ref{prop:laundering}.}
\label{fig:paths}
\end{figure}

\paragraph{Properties.}
We now state and prove the protocol's guarantees.

\begin{proposition}[Tighten safety]\label{prop:tighten-safety}
Any sequence of tighten transitions starting from caste $c_0$
results in a caste $c_k$ such that
$\mathit{actions}(c_k) \subseteq \mathit{actions}(c_0)$.
\end{proposition}

\begin{proof}
Each admitted tighten step $c_{j-1} \to c_j$ satisfies
$\mathit{actions}(c_j) \subsetneq \mathit{actions}(c_{j-1})$ by
Definition~\ref{def:classes}. Chaining the inclusions over
$j = 1, \dots, k$ gives
$\mathit{actions}(c_k) \subseteq \mathit{actions}(c_0)$ by
transitivity of $\subseteq$.
\end{proof}

\begin{proposition}[Bounded-relax containment]
\label{prop:relax-containment}
Fix an operator-reset baseline caste $c_\star$, an envelope
$E \subseteq \mathcal{A}$, and per-axis bounds
$\{B_\alpha\}_{\alpha \in E}$. The protocol admits the full
$k$-step sequence of relax transitions
$c_\star = c_0 \to c_1 \to \dots \to c_k$ that
relaxes only along axes in $E$ if and only if the per-axis
\emph{prefix invariant}
\[
\Delta_\alpha(c_\star, c_j) \le B_\alpha
\quad\text{for every }\alpha \in E\text{ and every prefix }j \le k
\]
holds. The per-axis bounds are independent: the bound $B_\alpha$ does
not constrain $B_\beta$ for $\beta \ne \alpha$. The baseline $c_\star$ is
the operator-declared baseline for the current epoch and is fixed
throughout it; it changes only when the operator declares a new signed
baseline (a \emph{reset}), which begins a new epoch with its own signed
Merkle root against which that epoch's records are audited. The bound is a
fixed limit on elevation above $c_\star$, not a counter that usage
depletes. The reference implementation fixes one such baseline per
envelope and audits each epoch against its declared baseline; dynamic
in-log rebaselining within a single epoch is not separately modelled
(\S\ref{sec:discussion}).
\end{proposition}

\begin{proof}
($\Leftarrow$) Suppose the prefix invariant holds. At step $j$, the
admission check of the bounded-relax path
(\S\ref{sec:protocol}, step~2) evaluates
$\Delta_\alpha(c_\star, c_j) \le B_\alpha$ for each $\alpha \in E$,
which holds by assumption, so every step is admitted.
($\Rightarrow$) Suppose some prefix violates the invariant, and let
$j$ be the first such prefix, with $\Delta_\beta(c_\star, c_j) > B_\beta$
for some axis $\beta$. The step-2 check at step $j$ evaluates exactly
this quantity and rejects, so the sequence is not admitted past
$c_{j-1}$. Independence is immediate because each axis is checked
against its own bound, and the charge $\Delta_\alpha(c_\star, \cdot)$ is
measured from the fixed epoch baseline $c_\star$, changing only when the
operator declares a new baseline for a subsequent epoch.
\end{proof}

\begin{proposition}[Caste-laundering prevention]
\label{prop:laundering}
Let $c_\star$ be the operator-reset baseline. A tighten transition
$c_\star \to c'$ (with $\mathit{actions}(c') \subsetneq
\mathit{actions}(c_\star)$) followed by a relax $c' \to c''$ is charged
$\Delta_\alpha(c_\star, c'')$ on every axis $\alpha$, the same charge as
a direct relax $c_\star \to c''$. The intervening tighten therefore
grants no budget headroom, and the relax is admitted only if
$\Delta_\alpha(c_\star, c'') \le B_\alpha$ for all $\alpha$.
\end{proposition}

\begin{proof}
By Definition~\ref{def:lattice} the budget charge for reaching $c''$ is
$\Delta_\alpha(c_\star, c'')$, computed from the reset baseline rather
than from the current caste $c'$. This value is independent of the path
taken to $c''$, so the tighten to $c'$ leaves it unchanged. Admission
then reduces to the prefix invariant of
Proposition~\ref{prop:relax-containment} evaluated at $c''$. Had the
charge instead been measured from $c'$, the adversary could
manufacture headroom by tightening far and relaxing back in budgeted
steps; anchoring at $c_\star$ removes exactly this freedom.
\end{proof}

\begin{proposition}[Tighten-path quorum]
\label{prop:tighten-quorum}
Tightening proceeds without operator countersignature but is not
unilateral. Under the authenticated-Byzantine fleet model of
\S\ref{sec:fleetmodel} (at most $f = \lfloor (N-1)/3 \rfloor$ Byzantine
members), a tighten record enters the audit chain only when at least
$\lceil (N+f+1)/2 \rceil$ members acknowledge it. No single rogue robot
can then forge a tighten record without the endorsement of at least one
honest peer.
\end{proposition}

\begin{proof}
A quorum of size $q = \lceil (N+f+1)/2 \rceil$ satisfies $2q > N + f$,
so any two quorums share more than $f$ members and therefore at least
one honest member, the standard quorum-intersection
argument~\citep{lamport1982byzantine}. Since each honest member
acknowledges a record only after verifying it is a valid tighten under
the lattice (Definition~\ref{def:classes}), an admitted record carries
at least one honest endorsement.
\end{proof}

Totally ordering conflicting bindings for the same robot (excluding a
fork) additionally requires an append-only per-robot head-extension
discipline over the hash-chained log of \S\ref{sec:fleetmodel}; we treat
that ordering discipline as a deployment detail and do not rely on it
for any property above.

\begin{corollary}[By-construction attack resistance]
\label{cor:attacks}
Under the model of \S\ref{sec:fleetmodel}, the protocol admits none of
the four threats of \S\ref{sec:threats}: caste laundering and
repeated-relaxation escalation are refused by
Propositions~\ref{prop:relax-containment} and~\ref{prop:laundering}
(the per-axis prefix invariant charged from $c_\star$); operator
impersonation is refused because a relax or operator-only admission
requires a signature that verifies against the operator's public key;
and cause-chain forgery is refused because the triggering observation
must carry a valid upstream signature (\S\ref{sec:causechain}). The
empirical admission rates of \S\ref{sec:eval} are consistent with these
bounds.
\end{corollary}

\subsection{Cause-Chain Construction}\label{sec:causechain}

Every caste reassignment is associated with a \emph{cause-chain
record} (CCR), a tuple of signed fields that documents the
event sufficient for an external auditor to reconstruct what
happened and why.

\begin{definition}[Cause-chain record]\label{def:ccr}
A CCR for a proposed reassignment is
\[
\mathit{CCR} = \langle
  \mathit{obs},\ h_i,\ c,\ c',\ t,\ \mathit{class},\ \sigma_i,\
  \sigma_O^{?},\ \pi_F \rangle
\]
where $\mathit{obs}$ is the triggering observation,
$h_i$ is the proposing robot's identity hash from the
operator-signed identity manifest, $(c, c')$ are the source and target
castes, $t$ is the timestamp, $\mathit{class}$ is one of
\{tighten, bounded-relax, operator-only\},
$\sigma_i$ is the proposing robot's signature, and $\sigma_O^{?}$
is the operator countersignature (present iff class is
bounded-relax or operator-only). The robot signature $\sigma_i$
covers $\mathit{obs}$, $h_i$, $(c, c')$, and $t$; the class is
derived from $(c, c')$ against the signed lattice and is recomputed
by verifiers rather than signed. The stored payload comprises these
signed fields plus the recorded class; the Merkle inclusion proof
$\pi_F$ that binds a record to the audit chain's epoch root is
recomputed by an auditor from the log rather than carried inside
the record.
\end{definition}

The triggering observation $\mathit{obs}$ is structured: for
battery-low events it is a signed sensor reading; for payload
exceedance it is a signed force-torque measurement; for
operational priority shift it is a signed operator
directive. The signature on $\mathit{obs}$ is by the originating
sensor or actor, not by the proposing robot. This separation
prevents threat 4 of \S\ref{sec:threats} (cause-chain forgery
by the proposing robot): a forged observation has no valid
upstream signature and fails verification.

\paragraph{Forensic reconstruction.}
The cause-chain record supports complete after-the-fact reconstruction
of any reassignment from the operator-signed material alone: the
operator's public key, the operator-signed identity manifest, and the
audit chain.

\begin{proposition}[Forensic reconstruction]\label{prop:forensic}
Given the operator's public key, the operator-signed identity manifest
(which binds every principal's public key, the privilege lattice, and the
fleet roster of size $N$), and the operator-signed audit chain, an
external auditor can reconstruct and verify, for any reassignment event
recorded in the chain, the full CCR including the triggering observation,
the proposing robot's identity, the caste transition, the class, the
countersignature path taken, and, for a tighten, the Byzantine-quorum
certificate that endorsed it.
\end{proposition}

\begin{proof}
The auditor first verifies the operator's signature over the identity
manifest with the operator's public key, and thereafter trusts only the
manifest's key bindings, lattice, and roster. The audit chain is an
append-only, hash-chained log that stores each admitted CCR's \emph{full
serialized payload} (\S\ref{sec:fleetmodel}), so the record is recovered
from the log itself, independent of which fleet members are still online.
The auditor recomputes the entry hash and confirms the hash-chain link to
detect tampering, verifies the Merkle inclusion proof $\pi_F$ against the
operator-signed epoch root, and verifies the embedded signatures
$\sigma_i$, $\sigma_O^{?}$ (when present), and the upstream signature on
$\mathit{obs}$, using only the manifest's public keys. For a tighten
record it additionally verifies the persisted quorum certificate: at least
$q = \lceil (N{+}f{+}1)/2 \rceil$ distinct signatures over the record core
from members of the manifest's roster, so a tighten cannot be attributed
to a quorum that never endorsed it. If any check fails the record is
rejected as forged; if all pass, the auditor has a verifiable account of
who did what and why. The reference implementation (\S\ref{sec:eval})
performs exactly this reconstruction and verification from the operator's
public key and the signed manifest.
\end{proof}

\subsection{Distributed Audit and Fork Exclusion}\label{sec:replication}

The forensic account above treats the audit chain as one logical log. In a
deployed fleet the log is replicated: each member keeps its own copy, and
the members must agree on a single ordered history despite up to $f$
Byzantine members. We give the replicated layer explicitly. Each member
holds a local, hash-chained replica. A record is proposed for the next
sequence \emph{position}; every honest member that accepts the record signs
a \emph{position-attestation} $\langle p,\ H(\mathit{rec}),\ \sigma_m
\rangle$ binding the position $p$ to the record hash. An honest member signs
\emph{at most one} hash per position: it records the hash it signed at each
position and refuses any request to sign a different hash there, so a
malicious proposer cannot induce an honest member to equivocate. This is the
one-signature-per-position invariant the agreement proof relies on, and the
implementation enforces it for a running node rather than assuming it
(\S\ref{sec:eval}). The invariant is maintained in volatile per-node state, so
the agreement guarantee is stated for the lifetime of a running node; making it
survive a crash and restart requires durable signed-position storage, which we
treat as out of scope (\S\ref{sec:discussion}). A
member commits the record at position $p$ once it holds attestations from a
quorum $q = \lceil (N{+}f{+}1)/2 \rceil$ of distinct members for the
\emph{same} hash, and it checks that the stored payload hashes to that
certified hash, so a certificate for one record cannot admit another. An
equivocating (faulty) member that signs two different hashes at one position
emits two conflicting attestations, each individually valid under its key.
This layer is a \emph{replication} primitive: it sequences and replicates the
cause-chain records that have already passed the governed-admission checks of
\S\ref{sec:protocol} (the robot and sensor signatures, the envelope, and, for a
tighten, the admission quorum). Its guarantee is agreement and fork exclusion
over those admitted records, not a second admission gate; a deployment places
it behind the admission checks rather than exposing the raw attest interface to
arbitrary clients.

\begin{proposition}[Agreement and fork exclusion]\label{prop:fork-exclusion}
Under reliable broadcast, at most $f$ Byzantine members, and honest members
that retain their signed-position state for the duration (no crash-restart that
loses it), (i) at most one record commits at each position, so every honest
replica commits an identical ordered log (a fork-free total order); and (ii) an
equivocating member is detected by every honest member that receives both of
its conflicting attestations, and neither forked record commits.
\end{proposition}

\begin{proof}
(i) Suppose two records with distinct hashes both committed at position $p$.
Each commit required $q$ distinct attestations for its hash. Two sets of
size $q = \lceil (N{+}f{+}1)/2 \rceil$ intersect in at least $N - 2(N-q) =
2q - N \ge f + 1$ members, hence in at least one honest member. But an honest
member attests at most one hash per position, so it cannot be in both sets:
contradiction. Thus at most one record commits per position, and since every
honest member applies the same commit rule to the same reliably-broadcast
attestations, all honest replicas commit the same sequence. (ii) An
equivocator's two attestations $\langle p, H_a, \sigma\rangle$ and
$\langle p, H_b, \sigma\rangle$ with $H_a \ne H_b$ both verify under the
member's registered key, so any honest member holding both has a
non-repudiable proof of equivocation. Neither $H_a$ nor $H_b$ can gather $q$
honest attestations (each honest member signs only its own single view), so
by (i) no forked record commits.
\end{proof}

Scope: this models reliable broadcast, one decree per position, and a
\emph{static} membership fixed for the epoch. Dynamic node join and leave
(reconfiguration), view changes, leader election, and liveness under
asynchrony are standard consensus concerns we do not re-derive: they are
abstracted by the reliable-broadcast and static-membership assumptions, and a
production deployment would layer a reconfiguration/view-change protocol
(Paxos/Raft/PBFT-style epochs) beneath the safety property established here.
The claim here is that safety property (agreement and fork exclusion) that the
governed audit chain requires, evaluated directly in \S\ref{sec:rq5}.

\section{Experiments}\label{sec:eval}

We evaluate the protocol on five properties: tightening latency (RQ1),
bounded-relax gating behaviour (RQ2), mechanical prevention of the four
threat patterns of \S\ref{sec:threats} (RQ3), quorum admission and
operator-key-only offline audit (RQ4), and the distributed replicated-log
layer with fork exclusion (RQ5). Evaluation is sim-primary supplemented by a
robot-class timing anchor at $N{=}10$.

\subsection{Setup}

The protocol is realised as a genuine reference implementation
(released with the paper): each principal holds a real Ed25519
key~\citep{bernstein2012ed25519}, signatures are produced and verified
with a standard cryptographic library, and the audit chain is a real
SHA-256 hash-chained log with per-epoch Merkle roots and inclusion
proofs~\citep{merkle1988signature}. Admission decisions are made by
these cryptographic checks and by the security envelope, so the attack
results below are produced by the mechanism rather than stipulated: a
forged operator countersignature fails a real elliptic-curve
verification, and a laundering sequence is charged against the real
envelope. Latency is estimated by a discrete-event network model layered
on the implementation: one-way message latency follows a log-normal
(median $1.8$~ms, heavy tail) with a first-order contention term that
grows with $N$, and the operator round trip is a log-normal with median
$85$~ms modelling a separate workstation. The per-signature verification
cost in the model is not guessed: it is \emph{measured} on the
evaluation host (Ed25519 verify $\approx 0.09$~ms, sign
$\approx 0.03$~ms) and used directly.

The evaluation covers fleets of $N \in \{10, 50, 100\}$ across three
scenarios: \emph{hospital logistics}, \emph{factory cell}, and
\emph{defence patrol}, each with three castes and its own reassignment
triggers (battery and human-presence; payload and queue depth; mission
state). A reassignment is exercised at up to $0.5$~Hz per robot, well
above the $0.1$~Hz floor of the target deployments, so it must complete
far inside the $2$~s inter-event budget. Every latency cell aggregates
$6000$ trials; the attack harness runs $500$ trials per pattern per
scenario. The whole evaluation is run under five seeds
($20260611$--$20260615$): every reported latency median moves by less
than $0.05$~ms across seeds, and the attack trials are pooled
($7500$ per pattern) with Wilson $95\%$ intervals.

A \emph{robot-class timing} operating point additionally models a
ten-robot, TurtleBot4-class deployment over a wireless-inspired channel
(heavier latency variance) with slower robot-class signature costs
(sign/verify $0.18/0.32$~ms) and the operator on a separate workstation;
its median ($12.5$~ms) sits inside the simulator's $N{=}50$ to $N{=}100$
envelope. All numbers are from the reference implementation and its
network model; a physical-robot campaign is future work
(\S\ref{sec:discussion}), and the contention term is first-order and
does not capture wireless broadcast-storm effects at large $N$.

\subsection{Tightening Latency (RQ1)}

We measure end-to-end latency from the originating sensor event to the
proposing robot's adoption of the target caste, for auto-tighten
transitions: the robot signs the request, waits for a quorum of
$\lceil (N{+}f{+}1)/2 \rceil = q$ acknowledgements, and verifies those
$q$ acknowledgement signatures before adopting. Median tightening latency
stays in the single-digit to low-double-digit millisecond range and grows
with $N$ as the quorum size $q$, the shared-medium contention, and the
proposer's $O(q)$ acknowledgement-verification cost all grow
(Table~\ref{tab:eval-tighten}, Figure~\ref{fig:results}(a)). It is
essentially scenario-independent, as expected: the scenarios differ in
what triggers a reassignment, not in the cost of admitting one, and the
medians move by under $0.05$~ms across the five seeds. The
robot-class timing point at $N{=}10$ has a median of $12.5$~ms,
between the simulator's $N{=}50$ ($11.3$~ms) and $N{=}100$ ($17.7$~ms)
points, with a heavier $p99$ tail ($17.3$~ms) from the slower robot-class
crypto and heavier wireless variance.

\begin{table}[!ht]
\centering
\small
\caption{Median tightening latency (ms) by fleet size; $6000$ trials per
cell. Values are for the primary seed $20260611$; every median moves by
$<0.05$~ms across the five seeds ($20260611$--$20260615$). The latency
includes the proposer's $O(q)$ verification of the quorum
acknowledgements. Latency is scenario-independent, so a single governed
row is reported; the governed $p99$ is $8.1/12.2/18.3$~ms at
$N{=}10/50/100$. The robot-class timing point models a ten-robot
wireless deployment ($p99 = 17.3$~ms).}
\label{tab:eval-tighten}
\begin{tabularx}{0.92\linewidth}{@{}lYYY@{}}
\toprule
                                   & $N{=}10$ & $N{=}50$ & $N{=}100$ \\
\midrule
Governed protocol                  & 6.0  & 11.3 & 17.7 \\
Ungoverned (no audit gate)         & 2.1  & 3.4  & 4.8 \\
Coordinator reallocation           & 4.5  & 6.9  & 9.9 \\
Robot-class timing ($N{=}10$)       & 12.5 & ---  & --- \\
\bottomrule
\end{tabularx}
\end{table}

Table~\ref{tab:eval-tighten} also reports two gate-removal ablations: an
ungoverned reassignment (one broadcast, immediate adoption) and a
coordinator-mediated reallocation without audit, in the lineage
of~\citep{nair2003roleallocation}. Both are faster because they skip the
audit quorum and its acknowledgement verification; the governance gate
adds $3.8/7.9/12.9$~ms over the ungoverned ablation at $N{=}10/50/100$,
two orders of magnitude below the $2$~s inter-event budget.

\subsection{Bounded-Relax Gating (RQ2)}

We drive a stream of bounded-relax requests through the real protocol,
each consuming a random per-axis magnitude from a pre-issued ratification
token that the operator reissues once it is exhausted. Whether a request
short-circuits or goes live is decided by the actual token state inside
\texttt{propose()}, not by a synthetic side calculation. Across the three
scenarios and five seeds, $70$--$72\%$ of requests complete via the
short-circuit path (a local token check, no operator round trip) while
the remaining $28$--$30\%$ take the live path and incur the operator
round trip (modelled at a median of $85$~ms). The protocol thus reserves
live operator attention for the token-exhausting minority and handles the
majority of relaxations without an interactive signature, while the
security envelope holds on every admitted relaxation regardless of the
token's state (Proposition~\ref{prop:relax-containment}). The split
tracks the token headroom as expected rather than being a fixed
artefact: under headrooms of $0.5$, $1.0$, and $2.0$ on the same
workload, the short-circuit fraction is $40\%$, $71\%$, and $87\%$. The
measured quantity here is the short-circuit/live split; the $85$~ms
live-path figure is the model's operator-round-trip parameter, not a
measured bounded-relax latency. We are explicit about what the artifact
does and does not exercise: both paths carry the same operator
countersignature over the record, and the ratification token models the
\emph{accounting} of a pre-issued authorization (whether the request is
covered without a live operator round trip). The artifact does not
implement a distinct cryptographic fast-path object separate from a live
signature; the $85$~ms is applied externally to the token-live requests
rather than emerging from two different signing primitives.

\subsection{Attack Pattern Resistance (RQ3)}

We run each of the four threats of \S\ref{sec:threats} as a randomized
adversarial test, not a single scripted case: each trial builds a fresh
world with fresh keys and randomized observation values and crafts the
genuinely-forged records the attack needs (a forged operator
countersignature signed with the adversary's key, an observation signed
by a non-sensor, a tighten-then-relax sequence against the envelope).
Refusal is decided by real Ed25519 verification and the envelope check,
never by a branch keyed on the attack name. The same crafted records are
run through three \emph{ablations} of the governed protocol, not
competing secure-MAS systems: an \emph{ungoverned} variant with no gate;
a \emph{reallocation} variant that authenticates the requester but does
not govern privilege; and an \emph{authorization-only} variant that
requires a valid operator countersignature for a relaxation but omits the
per-axis envelope and the observation check. We report, pooled over the
five seeds ($7500$ trials per pattern), the fraction \emph{admitted}
(target $0$) with its Wilson $95\%$ upper bound and the measured
per-attempt verification cost (Table~\ref{tab:eval-attacks},
Figure~\ref{fig:results}(b)). The governed protocol admits zero of $7500$
attempts for every pattern (Wilson upper bound $0.0005$). The measured cost of
the governed protocol's decisive \texttt{propose()} decision (its signature and
envelope checks, timed in isolation from attack construction and the comparator
baselines) is $0.19$--$0.30$~ms per attempt across the four patterns
(Table~\ref{tab:eval-attacks}). The comparison is informative rather than trivial:
the ungoverned and reallocation baselines admit every attack, but the
authorization-only baseline admits caste laundering, repeated-relax
escalation, and cause-chain forgery while \emph{refusing} operator
impersonation. This isolates which gate each attack requires:
impersonation is stopped by the operator-signature check alone, whereas
laundering and escalation need the per-axis envelope and forgery needs
the upstream-observation check. Because the governed protocol refuses all
four by construction (Corollary~\ref{cor:attacks}), its column is a
conformance check that the real implementation realises the proven
guarantee; the baseline gradient is the empirical content. The gate-free
and authorization-only baselines are ablations that isolate the effect of
removing a gate, not competitive secure-MAS systems.

\begin{table}[!ht]
\centering
\small
\caption{Attack admission rate (lower better; $7500$ pooled trials per
pattern over five seeds; Wilson $95\%$ upper bound in brackets). ``Gov.'' is
the governed protocol; ``Auth.'' is the authorization-only comparator (operator
countersignature required, no envelope or observation check);
``Gate-free'' is the ungoverned and reallocation baselines, which both
admit every attack. ``Cost'' is the measured wall-clock of the governed
protocol's decisive \texttt{propose()} decision for one attempt, excluding
attack construction and the comparator baselines.}
\label{tab:eval-attacks}
\begin{tabularx}{\linewidth}{@{}lYYYY@{}}
\toprule
\textbf{Pattern}          & \textbf{Gov.} & \textbf{Auth.} & \textbf{Gate-free} & \textbf{Cost (ms)} \\
\midrule
Caste laundering          & 0 ($<$0.001) & 1.00 & 1.00 & 0.20 \\
Repeated-relax escalation & 0 ($<$0.001) & 1.00 & 1.00 & 0.19 \\
Operator impersonation    & 0 ($<$0.001) & 0.00 & 1.00 & 0.30 \\
Cause-chain forgery       & 0 ($<$0.001) & 1.00 & 1.00 & 0.20 \\
\bottomrule
\end{tabularx}
\end{table}

\begin{figure}[!t]
\centering
\includegraphics[width=\linewidth]{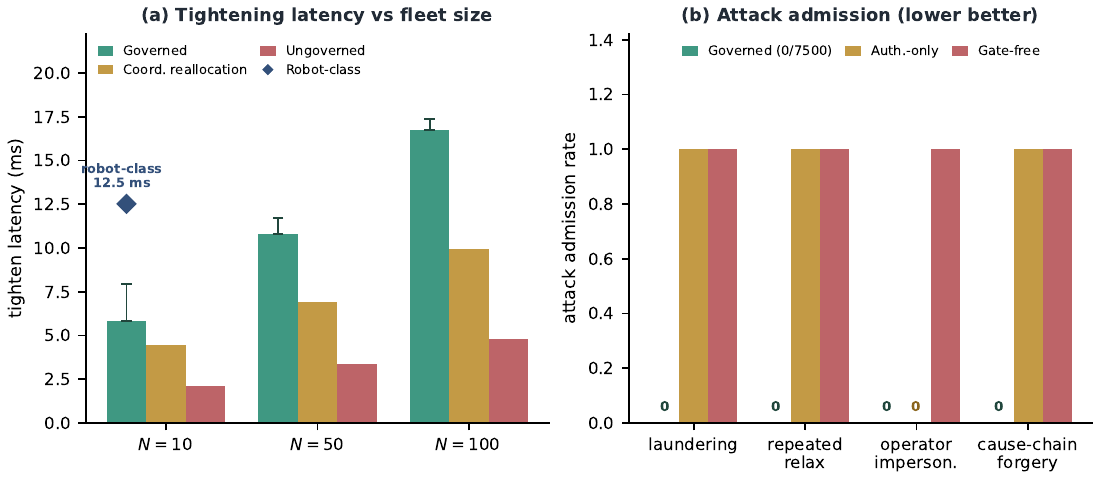}
\caption{(a) Median tightening latency versus fleet size for the governed
protocol and the two gate-removal ablations (the governed bar carries a whisker
to its $p99$), with the robot-class timing $N{=}10$ point marked; the
governance gate costs single-digit to low-double-digit milliseconds over
an ungoverned reassignment, growing with fleet size. (b) Attack-admission rate per pattern: the governed protocol
admits $0/7500$ (the Wilson $95\%$ whisker is below the axis resolution);
the authorization-only baseline admits all but operator impersonation; the
gate-removal ablations admit every attack.}
\label{fig:results}
\end{figure}

\paragraph{Randomized adversary (attack discovery).}\label{sec:rq3b} The four
patterns above are named threats. To probe the admission surface beyond them, we also run a
randomized adversary that, on each of $2000$ trials, draws a random mutation of
an otherwise-valid record (a forged robot signature, an observation signed by a
non-sensor, a corrupted identity hash, or a relaxation beyond the envelope with
a forged operator countersignature) and one legitimate tighten as a control.
The governed protocol admitted $0$ of the mutations it was not meant to admit
and every legitimate control, so no random malformation slipped through
(\texttt{fuzz\_adversary} in \texttt{results.json}). This is coverage over
random inputs rather than a fixed script, and it found no admission the proofs
do not already exclude.

\subsection{Quorum and Offline Audit (RQ4)}

Two further properties are exercised directly against the implementation.
First, tighten admission is subject to the Byzantine quorum of
\S\ref{sec:fleetmodel}, using real signed acknowledgements: each fleet
member that accepts the record returns an Ed25519 signature over the
record core, and the proposer admits only when at least $q$ valid
distinct member signatures verify. Over $1000$ trials per fleet size, an
honest tighten always reaches the quorum ($1.00$) while a forged tighten
(a bad robot signature) backed by all $f$ Byzantine members never does
($0.00$ at $N{=}10, 50, 100$), so no rogue coalition of size $\le f$
produces an admitted tighten (Proposition~\ref{prop:tighten-quorum}).
Second, the
offline auditor of Proposition~\ref{prop:forensic} reconstructs and
verifies the log from the operator's public key alone: given the
operator-signed identity manifest (which binds every principal's public
key, the privilege lattice, and the fleet roster of size $N$) and the
operator-signed epoch root, it recovers every record from the log,
verifies the hash-chain, the Merkle inclusion proofs, and every embedded
signature, and rejects a log in which a single stored payload byte has been
altered. Each tighten record also carries a persisted
\emph{quorum certificate}, the set of distinct member acknowledgement
signatures gathered at admission, so the auditor re-checks offline that the
tighten was endorsed by at least $q$ registered members; a tighten record
whose certificate is stripped or forged is rejected offline, closing the
gap between the admission-time quorum check and the forensic claim. Given
the operator-declared envelope, the auditor also recomputes each record's
class and rejects a bounded-relax record whose target exceeds the per-axis
bound of the reset baseline, so an out-of-envelope but validly
countersigned record is caught offline. It further confirms that each
record's identity hash $h_i$ is the one bound to the signer's registered
key, so a record cannot be attributed to a robot that did not sign it; a
registered adversary that signs a valid tighten while embedding a victim's
identity hash is rejected both at admission and offline.

\subsection{Distributed Audit and Fork Exclusion (RQ5)}\label{sec:rq5}

The experiments above run the audit chain as one logical log. To evaluate
the genuinely distributed layer, we replace that single log with $N$
per-member replicas and a quorum-committed replication protocol
(\S\ref{sec:replication}): a record commits at a sequence position only when
a quorum $q$ of members sign a \emph{position-attestation} binding
(position, record hash). Because any two quorums share an honest member, and
an honest member attests at most one hash per position, two conflicting
records cannot both reach quorum at one position; every honest replica
therefore commits the same ordered sequence. A Byzantine member that signs
two different hashes at one position produces two conflicting, individually
valid attestations: a cryptographic proof of equivocation that any honest
member holding both extracts.

We run this layer directly (\texttt{gcr/replication.py}) at
$N{=}10/50/100$. Table~\ref{tab:eval-distributed} reports, for each fleet
size, that all honest replicas converge on one total order (agreement),
the distributed-commit latency (two broadcast rounds plus the proposer's
$q$ attestation verifications), and the outcome of a Byzantine
equivocation. Across $50$ equivocation trials per fleet size, every honest
replica detects the fork ($1.00$) and no forked record ever commits
($0.00$): quorum intersection makes a conflicting commit impossible, and
the equivocation proof is recovered on every honest node. The
distributed-commit latency stays in the single-digit to low-double-digit
millisecond range, roughly the cost of the single-round tighten plus a
second broadcast round, so replication does not change the operating regime.

\begin{table}[!ht]
\centering
\small
\caption{Distributed audit over $N$ replicated per-member logs. ``Agree''
is $1$ when every honest replica commits an identical ordered log; ``Commit''
is the median distributed-commit latency (with $p99$); ``Detect'' is the
fraction of Byzantine-equivocation trials detected on every honest replica;
``Fork'' is the fraction of trials in which any honest replica committed a
forked record ($f$ is the tolerated Byzantine bound, $q$ the quorum).}
\label{tab:eval-distributed}
\begin{tabularx}{\linewidth}{@{}l YY Y Y Y@{}}
\toprule
$N$ ($f$, $q$) & Agree & Commit (ms) & $p99$ & Detect & Fork \\
\midrule
$10$ ($3$, $7$)   & $1.00$ & $5.5$  & $10.1$ & $1.00$ & $0.00$ \\
$50$ ($16$, $34$) & $1.00$ & $10.7$ & $15.8$ & $1.00$ & $0.00$ \\
$100$ ($33$, $67$)& $1.00$ & $17.0$ & $21.7$ & $1.00$ & $0.00$ \\
\bottomrule
\end{tabularx}
\end{table}

\paragraph{Real multi-process deployment.}\label{sec:rq5-real}
The table above uses the discrete-event network model to isolate the
protocol's own message and verification cost. To confirm the safety
properties on a genuinely distributed system rather than a model, we also run
the layer as a \emph{real deployment}: each of the $N$ members is a separate
operating-system process with its own TCP socket, its own Ed25519 key, and
its own replicated log (\texttt{netnode.py}), and a driver runs the real
two-round quorum-commit protocol over real sockets, with one member launched
as a real Byzantine equivocator process (\texttt{net\_experiment.py}). We ran
this at $N{=}10/50/100$ (up to $100$ concurrent processes) on a
$128$-core Linux host. Table~\ref{tab:eval-real} reports the outcome: at every
fleet size all honest processes agree on one ordered log, every honest process
detects the Byzantine equivocation, and no forked record commits. We also run
the harder \emph{malicious-proposer} attack with zero Byzantine members: the
proposer asks every honest process to sign two conflicting hashes at one
position. Every honest process (all $9/9$, $49/49$, $99/99$) signs the first
and refuses the second over the real socket, so the fork never gathers a quorum
and no forked record commits: the one-signature-per-position invariant holds on
the real deployment, not only in the proof. The commit
latency here is measured end-to-end over real TCP and includes interpreter,
loopback, and thread-pool overhead of an unoptimized single-host reference
(hence its growth with $N$ as the driver fans out to more sockets); it is a
real-system measurement, not the protocol's modelled cost. The point of this
run is that the agreement, fork-exclusion, and detection guarantees hold
unchanged on real processes over a real network, not that the reference driver
is latency-optimal.

\begin{table}[!ht]
\centering
\small
\caption{Real distributed deployment: $N$ fleet members as real OS processes
communicating over real TCP sockets, with a real Byzantine equivocator
process. ``Agree'', ``Detect'', and ``Fork'' are as in
Table~\ref{tab:eval-distributed}; ``Commit'' is the median (and $p99$)
end-to-end commit latency measured over real sockets on a $128$-core host
(an unoptimized single-host reference, so it grows with the driver's fan-out).}
\label{tab:eval-real}
\begin{tabularx}{\linewidth}{@{}l Y Y Y Y Y@{}}
\toprule
$N$ ($f$, $q$) & Agree & Commit (ms) & $p99$ & Detect & Fork \\
\midrule
$10$ ($3$, $7$)   & $1.00$ & $4.1$   & $6.4$   & $1.00$ & $0.00$ \\
$50$ ($16$, $34$) & $1.00$ & $30.1$  & $39.2$  & $1.00$ & $0.00$ \\
$100$ ($33$, $67$)& $1.00$ & $110.4$ & $193.9$ & $1.00$ & $0.00$ \\
\bottomrule
\end{tabularx}
\end{table}

\paragraph{Summary.} The evaluation supports the five claims of
\S\ref{sec:intro}: tightening completes in a single round trip without
operator involvement at single-digit to low-double-digit millisecond
latency; the audit gate adds single-digit to low-double-digit milliseconds
over an ungoverned reassignment; bounded relaxation completes via the short-circuit path
for the majority of requests while reserving live operator attention for
the token-exhausting minority; the governed protocol refuses all four
enumerated attacks by construction while a partially-governed baseline
isolates the gate each attack requires; the offline auditor and the
Byzantine quorum are exercised directly on the implementation; and the
distributed replicated-log layer commits a fork-free total order and detects
every Byzantine equivocation.

\subsection{Claim-to-Artifact Traceability}\label{sec:traceability}

Each formal claim is realised by a named module of the reference
implementation and checked by a named test; Table~\ref{tab:traceability}
maps them so a reviewer can locate the code and rerun the exact check
(\texttt{pytest -q} in the artifact runs all of them). The full reference
implementation, tests, and evaluation data are publicly available at
\url{https://github.com/s20sc/governed-caste-reassignment}.

\begin{table}[!ht]
\centering
\small
\caption{Mapping from each formal claim to the implementing module and the
test that exercises it. Module paths are under \texttt{gcr/}; the test
names below are the \texttt{test\_} functions in \texttt{tests/test\_gcr.py}
(the \texttt{test\_} prefix is omitted for space).}
\label{tab:traceability}
\newcommand{\us}{\_\allowbreak}
\begin{tabularx}{\linewidth}{@{}ll>{\raggedright\arraybackslash}X@{}}
\toprule
\textbf{Claim} & \textbf{Module} & \textbf{Test} \\
\midrule
Prop.~\ref{prop:tighten-safety} (tighten safety)    & \texttt{classes.py} & \texttt{honest\us transitions\us admitted} \\
Prop.~\ref{prop:relax-containment} (relax containment) & \texttt{budgets.py} & \texttt{offline\us verifier\us\dots} (out-of-envelope) \\
Prop.~\ref{prop:laundering} (laundering prevention) & \texttt{protocol.py}, \texttt{attacks.py} & \texttt{governed\us refuses\us all\us attacks\us\dots} \\
Prop.~\ref{prop:tighten-quorum} (tighten quorum)    & \texttt{quorum.py} & \texttt{quorum\us rejects\us forged\us tighten}; \texttt{propose\us tighten\us is\us quorum\us gated} \\
Cor.~\ref{cor:attacks} (all four refused)           & \texttt{protocol.py}, \texttt{attacks.py} & \texttt{governed\us refuses\us all\us attacks\us\dots}; \texttt{authorization\us only\us refuses\us only\us impersonation} \\
Prop.~\ref{prop:forensic} (forensic reconstruction) & \texttt{verify.py}, \texttt{manifest.py} & \texttt{offline\us verifier\us reconstructs\us and\us detects\us tampering}; \texttt{identity\us hash\us forgery\us rejected}; \texttt{manifest\us offline\us audit\us and\us quorum\us certificate} \\
Prop.~\ref{prop:fork-exclusion} (agreement, fork exclusion) & \texttt{replication.py}, \texttt{netnode.py} & \texttt{malicious\us proposer\us cannot\us fork\us honest\us nodes}; \texttt{f\us colluding\us byzantine\us cannot\us fork}; \texttt{byzantine\us equivocation\us detected\us and\us excluded}; \texttt{real\us multiprocess\us deployment\us agrees\us and\us excludes\us forks} \\
\bottomrule
\end{tabularx}
\end{table}

\section{Discussion and Limitations}\label{sec:discussion}

\paragraph{Operator Reachability.}
The bounded-relax path admits a short-circuit form using a
pre-issued ratification token bound to the operator's envelope.
When the operator network is partitioned or the operator is
otherwise unreachable, the short-circuit form continues to
admit requests within the most recent envelope, while requests
requiring a live countersignature stall. A time-bounded
waiting policy admits requests within a configured deadline
or auto-denies them; the safe fallback is to admit a
tightening request that lowers the robot to a known-safe
caste. The detailed policy is operator-configured and is part
of the operator's deployment-time commitment.

\paragraph{Operator Key Compromise.}
If the operator's signing key is compromised, the adversary can
forge countersignatures and admit arbitrary relaxations. The
construction relies on standard operator-key hygiene and does not
add an internal defence; standard mitigations (hardware-bound keys,
rotation, revocation) apply.

\paragraph{Pre-Committed Envelope Specification.}
The bounded-relax path requires the operator to specify, at
deployment time, the axes along which relaxation is allowed
and the bound along each axis. Specifying these envelopes is
a deployment-time engineering task; specifying them badly
either admits unintended relaxations or blocks safe ones. We
treat envelope specification as a configuration concern and
leave its formalisation to future work.

\paragraph{Evaluation Scope.}
The evaluation runs the real protocol: real Ed25519 signatures and a
real hash-chained, Merkle-committed audit log, with the full cryptographic
substrate (per-principal signing keys, identity binding, the hash-chained
Merkle log, offline verification from the operator-signed manifest, and
quorum-signature verification). The latency and attack experiments
(RQ1--RQ4) run the audit chain as a single logical log; the fully
replicated $N$-member layer, with quorum-committed total order and fork
exclusion, is implemented and evaluated directly in RQ5
(\S\ref{sec:rq5}), both in a property simulation and as a real deployment of
$N$ separate processes over TCP sockets with a real Byzantine equivocator, so
cross-replica agreement and equivocation are measured on a real distributed
system rather than assumed. Three things remain out of the deployment: the real
run is a single-host reference (all processes on one machine over loopback,
not a multi-host wide-area network); the one-signature-per-position state is
held in volatile per-node memory, so the agreement guarantee is for the
lifetime of a running node and a crash-restart that lost the state would need
durable signed-position storage (a fork under crash-recovery is otherwise
possible and is future work); and the replication layer assumes
reliable broadcast and static membership, leaving asynchronous-consensus
concerns (network asynchrony, view changes, and liveness under asynchrony) to
a standard consensus substrate beneath the safety property we establish. The
modelled-network experiments (RQ1--RQ4) also use a discrete-event latency
model whose per-signature verification cost is measured on the host and whose
contention term is first-order (it does not capture wireless broadcast-storm
effects at large $N$). The evaluation includes
a parameterized robot-class timing point at ten robots but no measurements
from physical robots. The governed attack results are zero-by-construction
(Corollary~\ref{cor:attacks}) confirmed by running the real implementation,
so the governed column verifies that the implementation realises the proven
guarantee rather than discovering it empirically; the measured baselines,
verification costs, randomized-fuzz outcomes, and distributed-equivocation
outcomes are the empirical content. The admission-surface security claims
are scoped to the four enumerated threats plus the randomized fuzz coverage;
remaining threats (signature replay, stale observations, reset abuse, and a
compromised authorized sensor) are future work. The artifact ships boundary
tests (\texttt{tests/test\_scope\_boundaries.py}) that pin the
implementation's current behaviour at these out-of-scope points (for
example, that a replayed signed record is still admitted and a
stale-timestamp observation is not rejected), so the scope statement is
checkable rather than only asserted. A physical-robot campaign on a
TurtleBot4-class fleet and a richer contention model are the natural next
steps.

\paragraph{Physical Compromise.}
The protocol governs the \emph{cooperative} surface of a reassignment: it
constrains which caste bindings peers and the operator will endorse and
audit, so a compromised robot cannot obtain an admitted, audited
privilege elevation without the operator. It cannot, however, prevent a
physically compromised robot from actuating its own motors or sensors
unilaterally outside the protocol; cryptographic authorisation bounds
what the fleet will \emph{recognise and record}, not what a rogue body can
physically attempt. Detecting and containing such physical deviation is
orthogonal and out of scope.

\paragraph{Out-of-Scope.}
We do not address cross-organisation federations with multiple
mutually distrusting operators (deferred to the federation
primitive's future work~\citep{aeros-p16}); physical capture
attacks against individual robots; or real-time control
safety properties orthogonal to caste-change authorisation.

\section{Conclusion}\label{sec:conclusion}

A regulated embodied deployment cannot treat caste reassignment
as an internal allocation decision: every privilege elevation
must be attributable to a human signature and verifiable from
the audit chain. The protocol of this paper partitions caste
changes into three classes (auto-tighten, bounded-relax,
operator-only) admitted through three different authorisation
paths, and produces a cause-chain record verifiable offline
against the operator's public key. The construction extends a
single-agent persona-mutation governance gate to the swarm
setting and is anchored in the federation primitive
of~\citep{aeros-p16}. Future work includes formalising
the pre-committed envelope specification, addressing
multi-operator federations, and extending the construction to
a multi-operator certificate hierarchy with regulator-issued
roots.

\backmatter

\bmhead{Data availability}
The reference implementation, the regression tests, and the evaluation data
that support the findings of this study are openly available at
\url{https://github.com/s20sc/governed-caste-reassignment}.

\bibliography{references}

\end{document}